\documentclass[letterpaper]{article}
\usepackage{aaai}
\usepackage{times}
\usepackage{helvet}
\usepackage{courier}
\nocopyright
\usepackage{graphicx}

\newcommand{\citep}[1]{\cite{#1}}
\newcommand{\citet}[1]{\citeauthor{#1}~\shortcite{#1}}

\usepackage[usenames,dvipsnames]{xcolor}
\usepackage{amssymb,amsfonts,amsbsy,amsmath,amsthm}
\theoremstyle{plain}
\newtheorem{thm}{Theorem}
\newtheorem{theorem}[thm]{Theorem}

\newtheorem{corollary}[thm]{Corollary}

\newtheorem{prop}[thm]{Proposition}

\theoremstyle{definition}

\newtheorem{defi}{Definition}
\newtheorem{definition}[defi]{Definition}

\DeclareMathOperator*{\argmax}{arg\,max}

\newcommand{\bm}[1]{\mathbf{#1}}
\DeclareMathAlphabet{\mathbfsf}{\encodingdefault}{\sfdefault}{bx}{n}
\def\P{P}

\def\x{x}
\def\xs{\bm{x}}
\def\X{X}
\def\Xs{\bm{X}}
\def\Xd{\mathcal{X}}
\def\xl{\mathsf{x}}

\def\ys{\bm{y}}

\def\Ys{\bm{Y}}
\def\Yd{\mathcal{Y}}
\def\yl{\mathsf{y}}

\def\zs{\bm{z}}

\def\Zs{\bm{Z}}
\def\Zd{\mathcal{Z}}

\def\es{\bm{e}}

\def\Es{\bm{E}}

\def\qs{\bm{q}}
\def\Q{Q}
\def\Qs{\bm{Q}}

\newcommand{\mln}[1]{#1}

\newcommand{\dom}{\mathbfsf{D}}

\newcommand{\s}{\mathtt{Smokes}}
\renewcommand{\c}{\mathtt{Cancer}}
\newcommand{\f}{\mathtt{Friends}}
\newcommand{\alice}{\mathsf{A}}
\newcommand{\bob}{\mathsf{B}}
\newcommand{\charlie}{\mathsf{C}} 

\newcommand{\guy}[1]{}
\newcommand{\mat}[1]{}

\makeatletter 
\newcommand\mynobreakpar{\par\nobreak\@afterheading} 
\makeatother

\begin{document}
\title{Tractability through Exchangeability: \\ A New Perspective on Efficient Probabilistic Inference}

\author{
Mathias Niepert\thanks{Both authors contributed equally to the paper. Guy Van den Broeck is also at KU~Leuven, Belgium.} \\
Computer Science and Engineering\\University of Washington, Seattle\\ \texttt{mniepert@cs.washington.edu}
\And 
Guy Van den Broeck\footnotemark[1]\\
Computer Science Department\\University of California, Los Angeles\\ \texttt{guyvdb@cs.ucla.edu}
}

\maketitle

\begin{abstract}
\begin{quote}
Exchangeability is a central notion in statistics and probability theory. The assumption that an infinite sequence of data points is exchangeable is at the core of Bayesian statistics. However, \emph{finite} exchangeability as a statistical property that renders probabilistic inference tractable is less well-understood. We develop a theory of finite exchangeability and its relation to tractable probabilistic inference. The theory is complementary to that of independence and conditional independence. We show that tractable inference in probabilistic models with high treewidth and millions of variables can be explained with the notion of finite (partial) exchangeability. We also show that existing lifted inference algorithms implicitly utilize a combination of conditional independence and partial exchangeability. 
\end{quote}
\end{abstract}

\section{Introduction}

Probabilistic graphical models such as Bayesian and Markov networks  explicitly represent \emph{conditional independencies} of a probability distribution with their structure~\citep{Pearl:1988,Lauritzen:1996,Koller:2009,darwiche2009modeling}. Their wide-spread use in research and industry  can largely be attributed to this structural property and their declarative nature, separating representation and inference algorithms. Conditional independencies often lead to a more concise representation and facilitate efficient algorithms for parameter estimation and probabilistic inference. It is well-known, for instance, that probabilistic graphical models with a tree structure admit efficient inference.
In addition to conditional independencies, modern inference algorithms exploit \emph{contextual independencies}~\citep{boutilier1996context} to speed up probabilistic inference.

The time complexity of classical probabilistic inference algorithms is exponential in the treewidth~\citep{robertson1986graph} of the graphical model. Independence and its various manifestations often reduce treewidth and treewidth has been used in the literature as the decisive factor for assessing the tractability of probabilistic inference~(cf.~\citet{Koller:2009}, \citet{darwiche2009modeling}).
However, recent algorithmic developments have shown that inference in probabilistic graphical models can be highly tractable, 
even in high-treewidth models without any conditional independencies.
For instance, lifted probabilistic inference algorithms~\citep{Poole2003,Kersting:2012} often perform efficient inference in densely connected graphical models with millions of random variables.
With the success of lifted inference, understanding these algorithms and their tractability on a more fundamental level has become a central challenge. The most pressing question concerns the underlying statistical principle that allows inference to be tractable in the absence of independence.

The present paper contributes to a deeper understanding of the statistical properties that render inference tractable.
We consider an inference problem tractable when it is solved by an efficient algorithm, running in time polynomial in the number of random variables.
The crucial contribution is a comprehensive theory  that relates the notion of finite partial \emph{exchangeability}~\citep{Diaconis:1980} to tractability.
One instance is full exchangeability where the distribution is invariant under variable permutations. 
We develop a theory of exchangeable decompositions that results in novel tractability conditions. 
Similar to conditional independence, partial exchangeability decomposes a probabilistic model so as to facilitate efficient inference. 
Most importantly, the notions of conditional independence and partial exchangeability are complementary, and when combined, define a much larger class of tractable models than the class of models rendered tractable by conditional independence alone.

Conditional and contextual independence are such powerful concepts because they are \emph{statistical} properties of 
the distribution, regardless of the representation used.
Partial exchangeability is such a statistical property that is independent of any representation, be it a joint probability table, a Bayesian network, or a statistical relational model.
We introduce novel forms of exchangeability, discuss their sufficient statistics, and efficient inference algorithms.
The resulting exchangeability framework allows us to state known liftability results as corollaries, providing a first statistical characterization of exact lifted inference.
As an additional contribution, we connect the semantic notion of exchangeability to \emph{syntactic} notions of tractability by showing that liftable statistical relational models have the required exchangeability properties due to their syntactic symmetries.
% We also show that this exchangeability can be detected efficiently knowing only the symmetries of the model. 
We thereby unify notions of lifting from the exact and approximate inference community into a common framework. 
%While the presented results apply to all discrete probabilistic models, we use Markov logic networks as an intuitive and important use-case to demonstrate the novel tractability results. 

%\guy{decomposition width vs.~treewidth}

%\guy{Maybe there is a third practical contribution, being able to perform exact lifted inference in propositional models.}

%\guy{We can now also give complexity bounds in terms of the size of the MLN/number of predicates, where previously this was unbounded.}

\section{A Case Study: Markov Logic}

The analysis of exchangeability and tractable inference is developed in the context of arbitrary discrete probability distributions, independent of a particular representational formalism. Nevertheless, for the sake of accessibility, we will provide examples and intuitions for Markov logic, a well-known statistical relational language that exhibits several forms of exchangeability.
Hence, after the derivation of the theoretical results in each section, we apply the theory to the problem of inference in Markov logic networks.
This also allows us to link the theory to existing results from the lifted probabilistic inference~literature.

\subsection{Markov Logic Networks} \label{s:mlns}

We first introduce some standard concepts from function-free first-order logic.
An \textit{atom} $\P(t_1, \dots , t_n)$ consists of a predicate $\P/n$ of 
arity $n$ followed by $n$ argument terms $t_i$, which are either \textit{constants}, $\{\alice,\bob,\dots\}$ or \textit{logical variables} $\{\xl, \yl, \dots\}$.
% A \emph{proposition} is an atom without arguments. 
A \emph{unary atom} has one argument and a \emph{binary atom} has two.
% A {\em literal} is an atom $a$ or its negation $\neg a$.  
% A {\em clause} is a disjunction over a finite set of literals.  
% A theory in \textit{conjunctive normal form} (CNF) is a conjunction of clauses.  
% Applying the \textit{substitution} $\{X/t\}$ to expression $\Sigma$ replaces the logical variable $X$ in $\Sigma$ by the term $t$ and is denoted by $\Sigma\{X/t\}$.
% We will assume that all logical variables are universally quantified.
A formula combines atoms with connectives (e.g., $\land$, $\Leftrightarrow$).
A formula is {\em ground} if it contains no logical variables. The groundings of a formula are obtained by instantiating the variables with particular constants.

Many statistical relational languages have been proposed in recent years~\citep{Getoor:2007,DeRaedt2008-PILP}. We will work with one such language, called \emph{Markov logic networks}~(MLN)~\citep{richardson2006markov}.
An MLN is a set of tuples $(w,f)$, where $w$ is a real number representing a weight and $f$ is a formula in first-order logic. 
% First-order logic formulas without a weight are called \textit{hard formulas} and correspond to formulas with an infinite weight.
% We will assume that all logical variables in $f$ are universally quantified.\footnote{We transform existential quantifiers into disjunctions.}
Let us, for example, consider the following MLN
\mln{ \begin{align}
1.3 \quad & \s(\xl) \Rightarrow \c(\xl) \label{f:smokescancer}\\
1.5 \quad & \s(\xl) \land \f(\xl,\yl) \Rightarrow \s(\yl) \label{f:smokesfriendsmokes}
\end{align}}
which states that smokers are more likely to (\ref{f:smokescancer})~get cancer and (\ref{f:smokesfriendsmokes})~be friends with other smokers.

Given a domain of constants $\dom$, a first-order MLN \(\Delta\) induces a {\em grounding}, which is the MLN obtained by replacing each formula in \(\Delta\) with all its groundings (using the same weight).
Take for example the domain $\dom = \{\alice,\bob\}$ (e.g., two people, Alice and Bob), the above first-order MLN represents the following grounding.
\mln{
  \begin{align*}
    1.3 \quad & \s(\alice) \Rightarrow \c(\alice)  \\
    1.3 \quad & \s(\bob) \Rightarrow \c(\bob)  \\
    1.5 \quad & \s(\alice) \land \f(\alice,\alice) \Rightarrow \s(\alice)  \\
    1.5 \quad & \s(\alice) \land \f(\alice,\bob) \Rightarrow \s(\bob)  \\
    1.5 \quad & \s(\bob) \land \f(\bob,\alice) \Rightarrow \s(\alice)  \\
    1.5 \quad & \s(\bob) \land \f(\bob,\bob) \Rightarrow \s(\bob)
  \end{align*}
}
This ground MLN contains eight different random variables, which correspond to all groundings of 
atoms \(\s(\xl)\), \(\c(\xl)\) and \(\f(\xl,\yl)\). This leads to a distribution over \(2^8\) possible worlds. The weight
of each world is the product of the expressions \(\exp(w)\), where \((w,f)\) is a ground MLN formula and \(f\) is satisfied by the world. 
% The weights of worlds that do not satisfy a hard formula are set to zero.
The probabilities of worlds are obtained by normalizing their weights.
Without loss of generality~\citep{Jha2010}, we assume that first-order formulas contain no constants.

\subsection{Lifted Probabilistic Inference}

The advent of statistical relational languages such as Markov logic has motivated a new class of \emph{lifted inference} algorithms~\citep{Poole2003}. These algorithms exploit the high-level structure and symmetries of the first-order logic formulas to speed up inference~\citep{Kersting:2012}.
Surprisingly, they perform tractable inference even in the absence of conditional independencies.
For example, when interpreting the above MLN as an (undirected) probabilistic graphical model, all pairs of random variables in $\{\s(\alice), \s(\bob),\dots\}$ are connected by an edge due to the groundings of Formula~\ref{f:smokesfriendsmokes}. The model has no conditional or contextual independencies between the $\s$ variables.
Nevertheless, lifted inference algorithms exactly compute its single marginal probabilities in time linear in the size of the corresponding graphical model~\citep{VdBIJCAI11}, scaling up to millions of random variables.

As lifted inference research makes algorithmic progress, the quest for the source of tractability and its theoretical properties becomes increasingly important.
For exact lifted inference, most theoretical results are based on the notion of domain-lifted inference~\citep{VdBNIPS11}.
\begin{definition}[Domain-lifted]
  Domain-lifted inference algorithms run in time polynomial in $|\dom|$.
\end{definition}
\noindent Note that domain-lifted algorithms can be exponential in other parameters, such as the number of formulas and predicates.
Our current understanding of exact lifted inference is that syntactic properties of MLN formulas permit domain-lifted inference~\citep{VdBNIPS11,JaegerStarAI12,taghipour2013completeness}. We will review these results where relevant.
Moreover, the (fractional) automorphisms of the graphical model representation have been related to lifted inference~\citep{niepertorbits,hai2012automorphism,noessner:2013,mladenov2013lifted}.
While there are deep connections between automorphisms and exchangeability~\citep{niepertorbits,niepert2013symmetry,hai2012automorphism,BuiHB12}, we refer these to future work.

\begin{figure}[t]
\begin{center}
\includegraphics[width=0.14\textwidth]{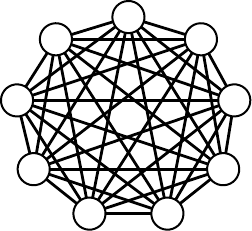}
\caption{\label{fig:symmetric} An undirected graphical model with $9$ finitely exchangeable Bernoulli variables. There are no (conditional) independencies that hold among the variables.}
\end{center}
\end{figure} 

\section{Finite Exchangeability}

%\guy{Introduce $\X = \x$ and $\Xs = \xs$ notation more explicitly?}
This section provides some background on the concept of finite partial exchangeability. We proceed by showing that particular forms of finite exchangeability permit tractable inference.
For the sake of simplicity and to provide links to statistical relational models such as MLNs, we present the theory for \emph{finite} sets of (upper-case) \emph{binary} random variables $\Xs = \{X_1,\dots,X_n\}$. However, the theory applies to all distributions over finite valued discrete random variables.
Lower-case $\xs$ denote an assignments to $\Xs$.

We begin with the most basic form of exchangeability.
\begin{definition}[Full Exchangeability]
\label{full-exch}
A set of variables $\Xs = \{X_1,...,X_n\}$ is \emph{fully exchangeable} if and only if 
$\Pr(\X_1 = \x_1,\dots,\X_n = \x_n) = \Pr(\X_1 = \x_{\pi(1)},\dots,\X_n = \x_{\pi(n)})$
for all permutations $\pi$ of $\{1, \dots, n\}$.
\end{definition}
\noindent Full exchangeability is best understood in the context of a finite sequence of binary random variables such as a number of coin tosses. 
Here, exchangeability means that it is only the number of heads that matters and not their particular order. Figure~\ref{fig:symmetric} depicts an undirected graphical model with $9$ finitely exchangeable dependent Bernoulli variables.

\subsection{Finite Partial Exchangeability}

The assumption that all variables of a probabilistic model are exchangeable is often too strong. Fortunately, exchangeability can be generalized to the concept of partial exchangeability using the notion of a \emph{sufficient statistic} \citep{DnF:1980,lauritzen:1984,Lauritzen:1988}. 
Particular instances of exchangeability such as full finite exchangeability correspond to particular statistics.

\begin{definition}[Partial Exchangeability]
Let  $\mathcal{D}_i$ be the domain of $X_i$, and let $\mathcal{T}$ be a finite set. A set of random variables $\Xs$ is  \emph{partially exchangeable} with respect to the statistic $T: \mathcal{D}_1\times\dots\times \mathcal{D}_n\rightarrow \mathcal{T}$ if and only if 
$$T(\xs)=T(\xs') \mbox { implies } \Pr(\xs)=\Pr(\xs').$$ 
\end{definition}

\noindent The following theorem states that the joint distribution of a set of random variables that is partially exchangeable with a statistic $T$ is a unique mixture of uniform distributions.

\begin{theorem}[\emph{\citet{Diaconis:1980}}]
  \label{theorem-param-pe}
  Let $\mathcal{T}$ be a finite set and let $T: \{0,1\}^n\rightarrow\mathcal{T}$ be a statistic of a partially exchangeable set $\Xs$.
  Moreover, let $S_t = \left\{\xs \in \mathcal{D}_1\times\dots\times \mathcal{D}_n \mid T(\xs)=t\right\}$, let $U_t$ be the uniform distribution over $S_t$, and let $\Pr(S_t) =  \Pr(T(\xs)=t)$. Then,
  \begin{align*}
  \Pr(\Xs) = \sum_{t \in \mathcal{T}} \Pr(S_t) U_t(\Xs).
  \end{align*}
\end{theorem}

\noindent Hence, a distribution that is partially exchangeable with respect to a statistic $T$ can be parameterized as a unique mixture of uniform distributions. We will see that several instances of partial exchangeability render probabilistic inference tractable. Indeed, the major theme of the present paper can be summarized as finding methods for constructing the above representation and exploiting it for tractable probabilistic inference for a given probabilistic model.

Let $[[\cdot]]$ be the indicator function. The uniform distribution of each equivalence class $S_t$ is $U_t(\Xs) = [[T(\Xs)=t]]/|S_t|$; and the probability of $S_t$ is $\Pr(S_t) = \Pr(\xs)  |S_t|$ for every $\xs \in S_t$. 
Hence, every value of the statistic $T$ corresponds to one equivalence class $S_t$ of joint assignments with identical probability. We will refer to these equivalence classes as \emph{orbits}.
We write $\xs \sim \es$ when assignments $\xs$ and $\es$ agree on the values of their shared variables~\citep{darwiche2009modeling}.
The \emph{suborbit} $S_{t,\es} \subseteq S_t$ for some evidence state $\es$ is the set of those states in $S_t$ that are compatible with $\es$, that is, $ S_{t,\es} = \left\{\xs \mid T(\xs) = t \text{ and } \xs \sim \es \right\}$.

\subsection{Partial Exchangeability and Probabilistic Inference}
\label{s:tractablestatistics}

We are now in the position to relate finite partial exchangeability to tractable probabilistic inference, using notions from Theorem~\ref{theorem-param-pe}.
The inference tasks we consider are
\begin{itemize}
  \item[--] \emph{MPE inference}, i.e., finding  $\argmax_{\ys} \Pr(\ys,\es)$ for any given assignment $\es$ to variables $\Es \subseteq \Xs$, and
  \item[--] \emph{marginal inference}, i.e., computing $\Pr(\es)$ for any given $\es$.
\end{itemize}
For a set of variables $\Xs$, we say that $P(\xs)$ can be computed \emph{efficiently} iff it can be computed in time polynomial in $|\Xs|$.
We make the following complexity claims
\begin{theorem}
  \label{theorem-eff}
  Let $\Xs$ be partially exchangeable with statistic $T$.
  If we can efficiently
  \mynobreakpar
  \begin{itemize}
    \item[--] for all $\xs$, evaluate $\Pr(\xs)$, and
    \item[--] for all $\es$ and $t \in \mathcal{T}$ decide whether there exists an $\xs \in S_{t,\es}$, and if so, construct it,
  \end{itemize}
  then the complexity of \emph{MPE inference} is polynomial in~$|\mathcal{T}|$.
  If we can additionally compute $|S_{t,\es}|$ efficiently, then the complexity of \emph{marginal inference} is also polynomial in $|\mathcal{T}|$.
\end{theorem}
  
\begin{proof}
  For MPE inference, we construct an $\xs_t \in S_{t,\es}$ for each $t \in \mathcal{T}$, and return the one maximizing $\Pr(\xs_t)$.
  For marginal inference, we return $\sum_{t \in \mathcal{T}} \Pr(\xs_t)|S_{t,\es}|$.
\end{proof}

If the above conditions for tractable inference are fulfilled we say that a distribution is \emph{tractably} partially exchangeable for MPE or marginal inference.
We will present notions of exchangeability and related statistics $T$ that make distributions tractably partially exchangeable.
Please note that Theorem~\ref{theorem-eff} generalizes to situations in which we can only efficiently compute $\Pr(\xs)$ up to a constant factor $Z$, as is often the case in undirected probabilistic graphical models.

\subsection{Markov Logic Case Study}

Exchangeability and independence are not mutually exclusive.
Independent and identically distributed (iid) random variables are also exchangeable.
Take for example the MLN
\mln{
  \begin{align*}
    1.5 \quad & \s(\xl)
  \end{align*}
}
The random variables $\s(\alice), \s(\bob), \dots$ are independent. Hence, we can compute their marginal probabilities independently as 
\begin{align*}
  \Pr(\s(\alice)) = \Pr(\s(\bob)) = \frac{\exp(1.5)}{\exp(1.5)+1}
\end{align*}
The variables are also finitely exchangeable. For example, the probability that $\alice$ smokes and $\bob$ does not is equal to the probability that $\bob$ smokes and $\alice$ does not.
The sufficient statistic $T(\xs)$ counts \emph{how many} people smoke in the state $\xs$ and the probability of a state in which $n$ out of $N$ people smoke is $\exp(1.5n)/(\exp(1.5)+1)^N$.

Exchangeability can occur without independence, as in the following Markov logic network
\mln{
  \begin{align*}
    1.5 \quad & \s(\xl) \land \s(\yl)
  \end{align*}
}
This distribution has neither independencies nor conditional independencies. 
However, its variables are finitely exchangeable and the probability of a state $\xs$ is only a function of the sufficient statistic $T(\xs)$ counting the number of smokers in $\xs$.
The probability of a state now increases by a factor of $\exp(1.5)$ with every \emph{pair of smokers}.
When $n$ people smoke there are $n^2$ pairs and, hence, $\Pr(\xs) = \exp\left(1.5 n^2\right)/Z$, where $Z$ is a normalization constant.
Let $\Ys$ consist of all $\s(\xl)$ variables except for $\s(\alice)$, and let $\ys$ be an assignment to $\Ys$ in which $m$ people smoke.
The probability 
that $\alice$ smokes given $\ys$ is
\begin{align*}
  \Pr\left(\s(\alice) \mid\ys\right) = \frac{\exp\left(1.5 (m+1)^2\right)}{\exp\left(1.5 (m+1)^2\right) + \exp\left(1.5 m^2\right)},
\end{align*}
which clearly depends on the number of smokers in $\ys$. Hence, $\s(\alice)$ is not independent of $\Ys$ but the random variables are exchangeable with sufficient statistic $n$. Figure~\ref{fig:symmetric} depicts the graphical representation of the corresponding ground Markov logic network. 

\section{Exchangeable Decompositions}

We now present novel instances of partial exchangeability that render probabilistic inference tractable. These instances generalize exchangeability of single variables to exchangeability of sets of variables. We describe the notion of an exchangeable decomposition and prove that it fulfills  the tractability requirements of Theorem~\ref{theorem-eff}.
We proceed by demonstrating that these forms commonly occur in MLNs.

\subsection{Variable Decompositions}

The notions of independent and exchangeable decompositions are at the core of the developed theoretical results. 

\begin{definition}[Variable Decomposition]
A \emph{variable decomposition} $\Xd = \{\Xs_1,\dots,\Xs_k\}$ partitions $\Xs$ into subsets $\Xs_i$.
We call $w = \max_{i} |\Xs_i|$ the \emph{width} of the decomposition.
\end{definition}

\begin{definition}[Independent Decomposition]
  \label{def:indep-decomp}
  A variable decomposition is \emph{independent} if and only if $\Pr$ factorizes~as
  \begin{equation*}
  \Pr\left(\Xs\right) = \prod_{i=1}^{k} Q_i(\Xs_i).
  \end{equation*}
%  where $Q_i(\xs_i)$ can be computed efficiently for all~$\xs_i~\in~\mathcal{D}_{\Xs_i}$.
\end{definition}

\guy{Say we assume that when there is independence, it is clear from whatever syntactic representation of $\Pr(\Xs)$ we have, what the $\Q_i(\Xs_i)$ are and how to evaluate them for full states $\xs_i$}

\begin{definition}[Exchangeable Decomposition]
  A variable decomposition is \emph{exchangeable} iff for all permutations $\pi$,
  \begin{align*}
    & \Pr\left(\Xs_1 = \xs_1,\dots,\Xs_k = \xs_k \right)  \\ 
    & \qquad\qquad = \Pr\left(\Xs_1 = \xs_{\pi(1)},\dots,\Xs_k = \xs_{\pi(k)}\right).
  \end{align*}
\end{definition}

\begin{figure}[t]
\begin{center}
\includegraphics[width=0.476\textwidth]{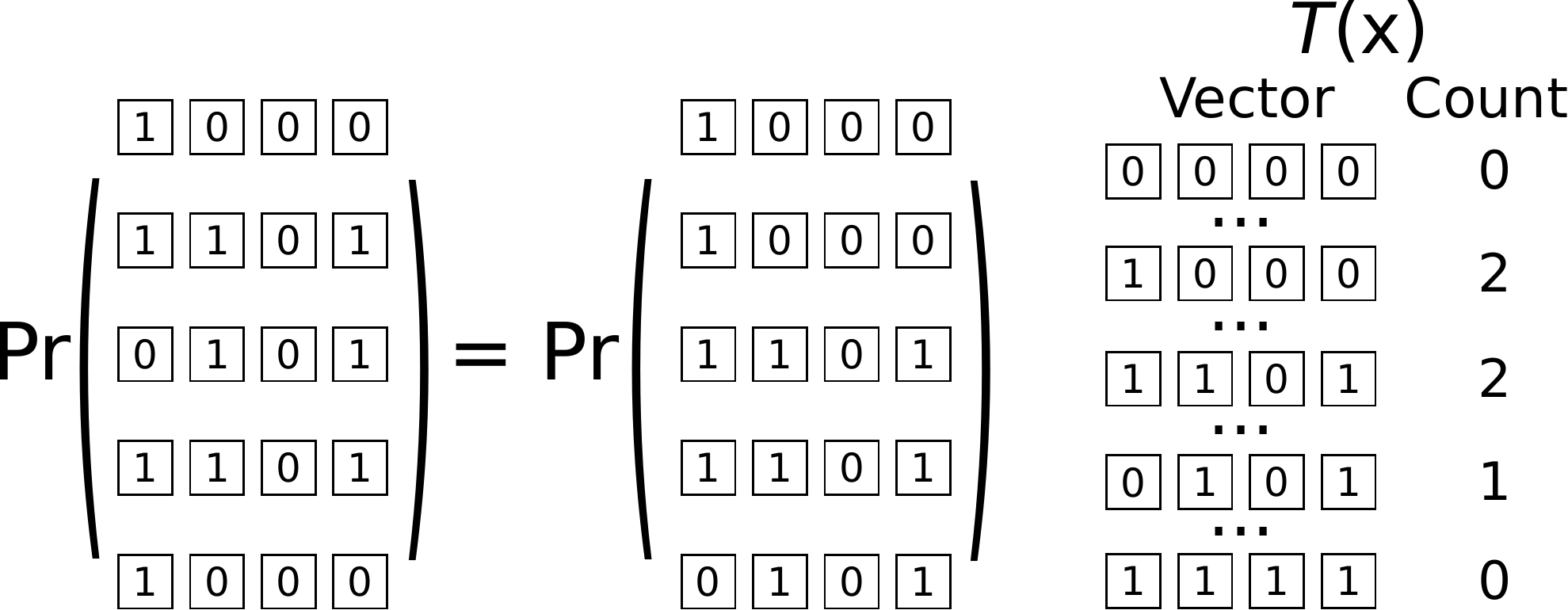}
\caption{\label{fig:generative} An exchangeable decomposition of $20$ binary random variables (the boxes) into $5$  components of width $4$ (the rows). The statistic $T$ counts the occurrences of each unique binary vector.}
\end{center}
\end{figure} 

Figure~\ref{fig:generative} depicts an example distribution $\Pr$ with $20$ random variables and a decomposition into $5$ subsets of width $4$. The joint distribution is invariant under permutations of the $5$ sequences. The corresponding sufficient statistic $T$ counts the number of occurrences of each binary vector of length $4$ and returns a tuple of counts.

\noindent Please note that the definition of full finite exchangeability~(Definition~\ref{full-exch}) is the special case when the exchangeable decomposition has width~$1$. Also note that the size of all subsets in an exchangeable decomposition equal the width.

\subsection{Tractable Variable Decompositions}

The core observation of the present work is that variable decompositions that are exchangeable and/or independent result in tractable probabilistic models. 
For independent decompositions, the following tractability guarantee is  used in most existing inference algorithms.
\begin{prop}
  \label{theorem-decomp-tractable-indep}
  Given an independent decomposition of $\Xs$ with bounded width, and a corresponding factorized representation of the distribution~(cf.~Definition~\ref{def:indep-decomp}), the complexity of MPE and marginal inference is polynomial in~$|\Xs|$.
\end{prop}
While the decomposition into independent components is a well-understood concept, the combination with finite exchangeability has not been previously investigated as a statistical property that facilitates tractable probabilistic inference.
We can now prove the following result.
\begin{theorem}
  \label{theorem-decomp-tractable}
  Suppose we can compute $\Pr(\xs)$ in time polynomial in $|\Xs|$.
  Then, given an exchangeable decomposition of $\Xs$ with bounded width, the complexity of MPE and marginal inference is polynomial in~$|\Xs|$.
\end{theorem}
\guy{Reviewer:
Generally, one shouldn't use mathematical symbols in theorem statements unless they are necessary. The symbol (X) for the decomposition is not referenced in the theorem and would be enough to state that there exists an exchangeable composition. (Same applies to theorems 7 and 8)}

\begin{proof}
Following Theorem~\ref{theorem-eff}, we have to show that there exists a statistic $T$ so that (a)~$|\mathcal{T}|$ is polynomial in $|\Xs|$; (b)~we can efficiently decide whether an $\xs \in S_{t,\es}$ exists and if so, construct it; and (c) efficiently compute $|S_{t,\es}|$ for all $\es$ and $t \in \mathcal{T}$. 
Statements (b) and (c) ensure that the assumptions of Theorem~\ref{theorem-eff} hold for exchangeable decompositions, which combined with (a) proves the theorem.

To prove (a), let us first construct a sufficient statistic for exchangeable decompositions.
A full joint assignment $\Xs = \xs$ decomposes into assignments $\Xs_1=\xs_1,\dots,\Xs_k =\xs_k$ in accordance with the given variable decomposition.
Each $\xs_\ell$ is a bit string $b \in \{0,1\}^w$.
Consider a statistic $T(\xs) = (c_1,\dots,c_{2^w})$, where each $c_i$ has a corresponding unique bit string $b_i \in \{0,1\}^w$ and $c_i = \sum_{\ell=1}^{k} [[\xs_\ell=b_i]]$.
The value $c_i$ of the statistic thus represents the number of components in the decomposition that are assigned bit string $b_i$.
Hence, we have $\sum_{i=1}^{2^w} c_i = k$, and we prove (a) by observing that
$$|\mathcal{T}| = \binom{k+2^{w}-1}{2^{w}-1} \leq k^{2^w - 1} \leq n^{2^w-1}.$$

To prove statements (b) and (c) we have to find, for each partial assignment $\Es = \es$, an algorithm that generates an $\xs \in S_{t,\es}$ and that computes $|S_{t,\es}|$ in time polynomial in $|\Xs|$. 
To hint at the proof strategy, we give the formula for the orbit without evidence $|S_t|$:
\begin{align*}
  \left|S_{(c_1,\dots,c_{2^w})}\right| = \binom{k}{c_1}\binom{k-c_1}{c_2}\dots\binom{k-\sum_{i=1}^{2^w-2}c_i}{c_{2^w-1}}. 
\end{align*}
The proof is very technical and deferred to the appendix.
\end{proof}

\subsection{Markov Logic Case Study}

Let us consider the following MLN
\mln{ \begin{align*}
  1.3 \quad & \s(\xl) \Rightarrow \c(\xl) \label{f:smokescancer}\\
  1.5 \quad & \s(\xl) \land \s(\yl)
\end{align*}}
It models a distribution in which every non-smoker or smoker with cancer, that is, every $\xl$ satisfying the first formula, increases the probability by a factor of $\exp(1.3)$. Each pair of smokers increases the probability by a factor of $\exp(1.5)$.
This model is not fully exchangeable: swapping $\s(\alice)$ and $\c(\alice)$ in a state yields a different probability. There are no (conditional) independencies between the $\s(\xl)$ atoms.

The variables in this MLN do have an exchangeable decomposition whose width is two, namely
\begin{align*}
  \Xd = \, 
    &\{\{\s(\alice),\c(\alice)\}, \\
    &\phantom{\{}\{\s(\bob),\c(\bob)\}, \\
    &\phantom{\{}\{\s(\charlie),\c(\charlie)\},\dots\}.
\end{align*}
The sufficient statistic of this decomposition counts the number of people in each of four groups, depending on whether they smoke, and whether they have cancer. 
The probability of a state only depends on the number of people in each group and swapping people between groups does not change the probability of a state. For example, 
\begin{align*}
  & \Pr(\s(\alice)=0,\c(\alice)=1, \\
  &   \qquad\qquad\qquad   \s(\bob)=1,\c(\bob)=0) \\
  & \quad = \Pr(\s(\alice)=1,\c(\alice)=0,\\
  & \qquad\qquad\qquad \s(\bob)=0,\c(\bob)=1).
\end{align*}
Theorem~\ref{theorem-decomp-tractable} says that this MLN permits tractable inference.

The fact that this MLN has an exchangeable decomposition is not a coincidence. In general, we can show this for MLNs of unary atoms, which are called \emph{monadic} MLNs.

\begin{theorem} \label{theorem:monadic-decomposition}
    The variables in a monadic MLN have an exchangeable decomposition. The width of this decomposition is equal to the number of predicates.
\end{theorem}
\noindent The proof builds on syntactic symmetries of the MLN, called renaming automorphisms~\citep{hai2012automorphism,niepert2012lmcmc}. Please see the appendix for further details.

It now follows as a corollary of Theorems~\ref{theorem-decomp-tractable} and~\ref{theorem:monadic-decomposition} that MPE and marginal inference in monadic MLNs is polynomial in $|\Xs|$, and therefore also in the domain size $|\dom|$.
\begin{corollary}
  \label{cor-unary-mln}
  Inference in monadic MLNs is domain-lifted.
\end{corollary}

%\guy{Do we relate to the lifted inference literature here, or in the discussion?}

\section{Marginal and Conditional Exchangeability}
%\guy{Version Guy}

Many distributions are not decomposable into independent or exchangeable decompositions. 
Similar to conditional independence, the notion of exchangeability can be extended to conditional exchangeability. 
We generalize exchangeability to conditional distributions, and state the corresponding tractability guarantees.

\subsection{Marginal and Conditional Decomposition}

Tractability results for exchangeable decompositions on all variables under consideration also extend to subsets.

\begin{definition}[Marginal Exchangeability]
When a subset $\Ys$ of the variables under consideration has an exchangeable decomposition $\Yd$, we say that $\Ys$ is \emph{marginally exchangeable}.
% Let $\Xs$ be a set of variables and let $\Ys \subseteq \Xs$. A decomposition $\Yd$ of variables $\Ys$ is \emph{marginally exchangeable} if and only if $\Yd$ is exchangeable. %Analogously, a decomposition $\Yd$ of variables $\Ys$ is \emph{marginally independent} if and only if   $\Yd$ is independent.
%   \begin{align*}
%     & \Pr\left(\Ys_1 = \ys_1,\dots,\Ys_k = \ys_k \right)  \\ 
%     & \qquad\qquad = \Pr\left(\Ys_1 = \ys_{\pi(1)},\dots,\Ys_k = \ys_{\pi(k)}\right)
%   \end{align*}
%   \guy{
%   or equivalently,
%   \begin{align*}
%     & \sum_\Zs \Pr\left(\Ys_1 = \ys_1,\dots,\Ys_k = \ys_k, \Zs \right) \\ 
%     & \qquad\qquad = \sum_\Zs \Pr\left(\Ys_1 = \ys_{\pi(1)},\dots,\Ys_k = \ys_{\pi(k)}, \Zs\right)
%   \end{align*}}
%   for all permutations $\pi$ of $\{1, \dots, k\}$.
%  \guy{marginal independence does not fit anymore into this definition, as it requires that $Q_i(\Ys_1)$ can efficiently be evaluated, which is what we want to prove, not define}
\end{definition}

\noindent This means that $\Yd$ is still an exchangeable decomposition when considering the distribution $\Pr(\Ys) = \sum_{\Xs \setminus \Ys} \Pr(\Xs)$.

\begin{theorem} \label{theorem:margdecomp-tractable}
Suppose we are given a marginally exchangeable decomposition of $\Ys$ with bounded width and let $\Zs = \Xs \setminus \Ys$.
If computing $\Pr(\ys) = \sum_\zs \Pr(\ys,\zs)$ is polynomial in $|\Xs|$ for all $\ys$, then the complexity of MPE and marginal inference over variables $\Ys$ is polynomial in $|\Xs|$.
\end{theorem}

\begin{proof}
Let $T$ be the statistic associated with the given decomposition, and let $\es$ be evidence given on $\Es \subseteq \Ys$. Then, 
\[P(\es) = \sum_{t \in \mathcal{T}} |S_{t,\es}| \cdot \sum_{\zs}\Pr(\ys_{t,\es},\zs), \text{~~where~~} \ys_{t,\es} \in S_{t,\es}.\]
By the assumption that $\Ys$ is marginally exchangeable and the proof of Theorem~\ref{theorem-decomp-tractable}, we can compute $|S_{t,\es}|$ and an $\ys_{t,\es} \in S_{t,\es}$ in time polynomial in~$|\Xs|$. An analogous argument holds for MPE inference on~$\Pr(\Ys)$.
\end{proof}

We now need to identify distributions $\Pr(\Ys,\Zs)$ for which we can compute $\Pr(\ys)$ efficiently. This implies tractable probabilistic inference over $\Ys$.
Given a particular $\ys$, we have already seen sufficient conditions: when $\Zs$ decomposes exchangeably or independently conditioned on $\ys$, Proposition~\ref{theorem-decomp-tractable-indep} and Theorem~\ref{theorem-decomp-tractable} guarantee that computing $\Pr(\ys)$ is tractable.
%\guy{shall we not mention the difference between marginals and probability of evidence?}
This suggests the following general notion.

\begin{definition}[Conditional Decomposability] 
Let $\Xs$ be a set of variables with $\Ys \subseteq \Xs$ and $\Zs = \Xs \setminus \Ys$.
We say that $\Zs$ is \emph{exchangeably (independently) decomposable given $\Ys$} 
if and only if for each assignment $\ys$ to $\Ys$, there exists an exchangeable (independent) decomposition $\Zd_\ys$ of~$\Zs$.
\end{definition}

\noindent Furthermore, we say that $\Zs$ is decomposable with \emph{bounded width} iff the width $w$ of each $\Zd_\ys$ is bounded.
When the decomposition can be computed in time polynomial in $|\Xs|$ for all $\ys$, we say that $\Zs$ is \emph{efficiently decomposable}.

\begin{theorem} \label{theorem:conditional-tractable}
Let $\Xs$ be a set of variables with $\Ys \subseteq \Xs$ and $\Zs = \Xs \setminus \Ys$.
Suppose we are given a marginally exchangeable decomposition of $\Ys$ with bounded width. 
Suppose further that $\Zs$ is efficiently (exchangeably or independently) decomposable given $\Ys$ with bounded width. If we can compute $\Pr(\xs)$ efficiently, then the complexity of MPE and marginal inference over variables $\Ys$ is polynomial in $|\Xs|$.
\end{theorem}

\begin{proof}
Following Theorem~\ref{theorem:margdecomp-tractable}, we only need to show that we can compute $\Pr(\ys) = \sum_\zs \Pr(\ys,\zs)$ in time polynomial in $|\Xs|$ for all $\ys$. When $\Zs$ is exchangeably decomposable given $\Ys$, this follows from constructing $\Zd_\ys$ and employing the arguments made in the proof of Theorem~\ref{theorem-decomp-tractable}.
The case when $\Zs$ is independently decomposable is analogous.
\end{proof}

Theorems~\ref{theorem:margdecomp-tractable} and~\ref{theorem:conditional-tractable} are powerful results and allow us to identify numerous probabilistic models for which inference is tractable. For instance, we will prove liftability results for Markov logic networks. 
However, we are only at the beginning of leveraging these tractability results to their fullest extent. Especially Theorem~\ref{theorem:margdecomp-tractable} is widely applicable because the computation of $\sum_\zs \Pr(\ys,\zs)$ can be tractable for many reasons. For instance, conditioned on the variables $\Ys$, the distributions $\Pr(\ys,\Zs)$ could be bounded treewidth graphical models, such as tree-structured Markov networks. Tractability for $\Pr(\Ys)$ follows immediately from Theorem~\ref{theorem:margdecomp-tractable}.

Since Theorem~\ref{theorem:margdecomp-tractable} only speaks to the tractability of querying variables in $\Ys$, there is the question of when we can also efficiently query the variables $\Zs$.
Results from the lifted inference literature may provide a solution by bounding or approximating queries and evidence that includes $\Zs$ to maintain marginal exchangeability~\citep{VdBDarwiche13}. The next section shows that certain restricted situations permit tractable inference on the variables in $\Zs$.

\subsection{Markov Logic Case Study}
Let us again consider the MLN
\mln{ \begin{align*}
  1.3 \quad & \s(\xl) \Rightarrow \c(\xl)\\
  1.5 \quad & \s(\xl) \land \f(\xl,\yl) \Rightarrow \s(\yl) 
\end{align*}}
having the marginally exchangeable decomposition 
\[\Yd = \{\{\s(\alice)\},\{\s(\bob)\},\{\s(\charlie)\},\dots\}\]
whose width is one.
To intuitively see why this decomposition is marginally exchangeable, let us consider two states $\ys$ and $\ys'$ of the $\s(\xl)$ atoms in which \emph{only} the values of two atoms, for example $\s(\alice)$ and $\s(\bob)$, are swapped.
There is a symmetry of the MLNs joint distribution that  swaps these atoms: the renaming automorphism~\citep{hai2012automorphism,niepert2012lmcmc} that swaps constants $\alice$ and $\bob$ in all atoms.
For marginal exchangeability, we need that $\sum_\zs \Pr(\ys,\zs)=\sum_\zs \Pr(\ys',\zs)$. But this holds since the renaming automorphism is an automorphism of the set of states $\{\ys\zs \mid \zs \in \mathcal{D}_\Zs\}$ -- for every $\ys$, $\ys'$, and $\zs$ there exists an automorphism that maps $\ys\zs$ to $\ys'\zs'$ with $\Pr(\ys,\zs)= \Pr(\ys',\zs')$.

The given MLN has several marginally exchangeable decompositions, with the most general one being 
\begin{align*}
  \Yd = \, 
    &\{\{\s(\alice),\c(\alice),\f(\alice,\alice)\}, \\
    &\phantom{\{}\{\s(\bob),\c(\bob),\f(\bob,\bob)\}, \\
    &\phantom{\{}\{\s(\charlie),\c(\charlie),\f(\charlie,\charlie)\},\dots\}.
\end{align*}
For that decomposition, the remaining $\Zs$ variables $$ \{\f(\alice,\bob),\f(\bob,\alice),\f(\alice,\charlie),\dots\}$$ are independently decomposable given $\Ys$. The $\Zs$ variables appear at most once in any formula. In a probabilistic graphical model representation, evidence on the $\Ys$ variables would therefore decompose the graph into independent components.
Thus, it follows from Theorem~\ref{theorem:conditional-tractable} that we can efficiently answer any query over the variables in $\Ys$.

This insight generalizes to a large class of MLNs, called the \emph{two-variable fragment}. It consists of all MLNs whose formulas contain at most two logical variables.
\begin{theorem} \label{theorem:twologvar}
In a two-variable fragment MLN, let $\Ys$ and $\Zs$ be the ground atoms with one and two distinct arguments respectively.
Then there exists a marginally exchangeable decomposition of $\Ys$, and $\Zs$ is efficiently independently decomposable given $\Ys$.
Each decomposition's width is at most twice the number of predicates.
\end{theorem}

The proof of Theorem~\ref{theorem:twologvar} is rather technical and we refer the reader to the appendix for a detailed proof.
It now follows from Theorems~\ref{theorem:conditional-tractable} and~\ref{theorem:twologvar} that the complexity of inference over the unary atoms in the two-variable fragment is polynomial in the domain size $|\dom|$.

What happens if our query involves variables from $\Zs$ -- the binary atoms? It is known in the lifted inference literature that we cannot expect efficient inference of general queries that involve the binary atoms. Assignments to the $\Zs$ variables break symmetries and therefore break marginal exchangeability. This causes inference to become \#P-hard as a function of the query~\citep{VdBAAAI12}.
Nevertheless, if we bound the number of binary atoms involved in the query, we can use the developed theory to show a general liftability result.

\begin{theorem} \label{theorem:twologvar-lifted}
For any MLN in the two-variable fragment, MPE and marginal inference over the unary atoms and a bounded number of binary atoms is domain-lifted.
\end{theorem}

\noindent This corresponds to one of the strongest known theoretical results in the lifted inference literature~\citep{JaegerStarAI12}.
We refer the interested reader to the appendix for the proof.
A consequence of Theorem~\ref{theorem:twologvar-lifted} is that we can efficiently compute \emph{all single marginals} in the two-variable fragment, given arbitrary evidence on the unary atoms.

\section{Discussion and Conclusion}

We conjecture that the concept of (partial) exchangeability has potential to contribute to a deeper understanding of tractable probabilistic models. The important role conditional independence plays in the research field of graphical models is evidence for this hypothesis. Similar to conditional independence, it might be possible to  develop a theory of exchangeability that mirrors that of independence. For instance, there might be a (graphical) structural representations of particular types of partial exchangeability and corresponding logical axiomatizations~\citep{Pearl:1988}. Moreover, it would be interesting to develop graphical models with exchangeability and independence, and notions like d-separation to detect marginal exchangeability and conditional decomposability from a structural representation. The first author has taken steps in this direction by introducing exchangeable variable models, a  class of (non-relational) probabilistic models based on finite partial exchangeability~\cite{niepert:2014b}.

Recently, there has been considerable interest in computing and exploiting the automorphisms of graphical models~\citep{niepertorbits,hai2012automorphism}. There are several  interesting connections between automorphisms, exchangeability, and lifted inference~\citep{niepert2012lmcmc}. Moreover, there are several group theoretical algorithms that one could apply to the automorphism  groups to discover the structure of exchangeable variable decompositions from the structure of the graphical models. Since we presently only exploit renaming automorphisms, there is a potential for tractable inference in MLNs that goes beyond what is known in the lifted inference literature.

Partial exchangeability is related to collective graphical models~\citep{Sheldon:2011} (CGMs) and cardinality-based potentials~\cite{Gupta:2007} as these models also operate on sufficient statistics. However, probabilistic inference for CGMs is not tractable and there are no theoretical results that identify tractable CGMs models. The presented work may help to identify such situations. The presented theory generalizes the statistics of cardinality-based potentials.

\subsubsection*{Lifted Inference and Exchangeability}
Our case studies identified a deep connection between lifted probabilistic inference and the concepts of partial, marginal and conditional exchangeability.
In this new context, it appears that exact lifted inference algorithms~\citep{DeSalvoBraz2005,milch2008lifted,Jha2010,VdBIJCAI11,gogate2012probabilistic,taghipour2012lifted} can all be understood as performing essentially three steps: 
(i)~construct a sufficient statistic $T(\xs)$,
(ii)~generate all possible values of the sufficient statistic, and
(iii)~count suborbit sizes for a given statistic.
%\guy{Is the following too much?}
For an example of (i), we can show that a compiled first-order circuit~\citep{VdBIJCAI11} or the trace of probabilistic theorem proving~\citep{gogate2012probabilistic} encode a sufficient statistic in their existential quantifier and splitting nodes.
Steps (ii) and (iii) are manifested in all these algorithms through summations and binomial coefficients.

% \paragraph{Liftability}
Between Corollary~\ref{cor-unary-mln} and Theorem~\ref{theorem:twologvar-lifted}, we have re-proven almost the entire range of liftability results from the lifted inference literature~\citep{JaegerStarAI12} within the exchangeability framework, and extended these to MPE inference.
There is an essential difference though: liftability results make assumptions about the syntax (e.g., MLNs), whereas our exchangeability theorems apply to all distributions. We expect Theorem~\ref{theorem:conditional-tractable} to be used to show liftability, and more general tractability results for many other representation languages, including but not limited to the large number of statistical relational languages that have been proposed~\citep{Getoor:2007,DeRaedt2008-PILP}.

%\section{Conclusions}

\section*{Acknowledgments}
This work was partially supported by ONR grants
\#N00014-12-1-0423, \#N00014-13-1-0720, and \#N00014-12-1-0312; NSF grants \#IIS-1118122 and \#IIS-0916161; ARO grant \#W911NF-08-1-0242; AFRL contract \#FA8750-13-2-0019; the Research Foundation-Flanders (FWO-Vlaanderen); and a Google research award to MN. %The views and conclusions contained in this document are those of the authors and should not be interpreted as necessarily representing the official policies, either expressed or implied, of ARO, ONR, AFRL, or the United States Government.

% \newpage
% \bibliographystyle{abbrv}

% \newpage

% \title{Tractability through Exchangeability: \\ A New Perspective on Efficient Probabilistic Inference\\
% ~\\
% \textsc{Appendix}}
% \date{}
% \maketitle
% \pagenumbering{gobble}% Remove page numbers (and reset to 1)

\appendix

\section{Continued Proof of Theorem~\ref{theorem-decomp-tractable}}

\begin{proof}

To prove statements (b) and (c), we need to represent partial assignments $\Es=\es$ with $\Es \subseteq \Xs$. The partial assignments decompose into partial assignments $\Es_1=\es_1,\dots,\Es_k = \es_k$ in accordance with $\mathcal{X}$.
Each $\es_\ell$ corresponds to a string $m \in \{0,1,*\}^w$ where 
characters $0$ and $1$ encode assignments to variables in $\Es_\ell$ and $*$ encodes an unassigned variable in $\Xs_\ell-\Es_\ell$.
In this case, we say that $\es_\ell$ is of type $m$. Please note that there are $2^w$ distinct $b$ and $3^w$ distinct $m$. We say that $\xs$ \emph{agrees with} $\es$, denoted by $\xs \sim \es$, if and only if their shared variables have identical assignments. 

A \emph{completion} $\mathtt{c}$ \emph{of} $\es$ \emph{to} $\xs$ is a bijection $\mathtt{c}: \{\es_1,...,\es_k\} \rightarrow \{\xs_1,...,\xs_k\}$ such that $\mathtt{c}(\es_i) = \xs_j$ implies $\es_i \sim \xs_j$. Every completion corresponds to a unique way to assign elements in $\{0,1\}$ to unassigned variables so as to turn the partial assignment $\es$ into the full assignment $\xs$.

Let $t = (c_1,...,c_{2^w}) \in \mathcal{T}$, let $\Es \subseteq \Xs$, and let $\es \in \mathcal{D}_{\Es}$. Moreover, let $d_j = \sum_{\ell=1}^{k} [[\es_\ell=m_j]]$ for each $m_j \in \{0,1,*\}^w$. Consider the set of matrices
\begin{align*}
  \mathcal{A}_{t,\es} = \Big\{A \in \mathbb{M}(2^w,3^w) ~\Big|~ \sum_{i} a_{i,j} = d_j, \sum_{j} a_{i,j} = c_i \\
  \mbox{ and } a_{i,j} = 0 \mbox{ if } b_i \not\sim m_j\Big\}.
\end{align*}
Every $A \in \mathcal{A}_{t,\es}$ represents a set of completions from $\es$ to $\xs$ for which $T(\xs) = t$. The value $a_{i,j}$ indicates that each completion represented by $A$ maps $a_{i,j}$ elements in $\{\es_1,...,\es_k\}$ of type $m_j$ to $a_{i,j}$ elements in $\{\xs_1,...,\xs_k\}$ of type $b_i$. We write $\gamma(A)$ for the set of completions represented by $A$.

We have to prove the following statements
\begin{enumerate}
\item For every $A \in \mathcal{A}_{t,\es}$ and every $\mathtt{c} \in \gamma(A)$ there exists an $\xs \in \mathcal{D}_{\Xs}$ with $T(\xs)=t$ and $\mathtt{c}$ is a completion of $\es$ to $\xs$;
\item For every $\xs \in \mathcal{D}_{\Xs}$ with $T(\xs) = t$ and every  completion $\mathtt{c}$ of $\es$ to $\xs$ there exists an $A \in \mathcal{A}_{t,\es}$ such that $\mathtt{c} \in \gamma(A)$;
\item For all $A, A' \in \mathcal{A}_{t,\es}$ with $A \neq A'$ we have that $\gamma(A) \cap \gamma(A') = \emptyset$;
\item For every  $A \in \mathcal{A}_{t,\es}$, we can efficiently compute $|\gamma(A)|$, the size of the set of completions represented by $A$.
\end{enumerate}

To prove statement (1), let $A \in \mathcal{A}_{t,\es}$ and $\mathtt{c} \in \gamma(A)$. By the definition of $\mathcal{A}_{t,\es}$, $\mathtt{c}$ maps $a_{i,j}$ elements in $\{\es_1,...,\es_k\}$ of type $m_j$ to $a_{i,j}$ elements in $\{\xs_1,...,\xs_k\}$ of type $b_i$. By the conditions $\sum_{i} a_{i,j} = d_j$ and $\sum_{j} a_{i,j} = c_i$ of the definition of $\mathcal{A}_{t,\es}$ we have that $\mathtt{c}$ is a bijection. By the condition $a_{i,j} = 0 \mbox{ if } b_i \not\sim m_j$ of the definition of $\mathcal{A}_{t,\es}$ we have that $\mathtt{c}(\es_i) = \xs_j$ implies $\es_i \sim \xs_j$ and, therefore, $\es \sim \xs$. Hence, $\mathtt{c}$ is a completion. Moreover, $\mathtt{c}$ completes $\es$ to an $\xs$ with $\sum_{\ell=1}^{k} [[\xs_\ell=b_i]] = \sum_{j} a_{i,j} = c_i$ by the definition of $\mathcal{A}_{t,\es}$. Hence, $T(\xs) = t$.

To prove statement (2), let $\xs \in \mathcal{D}_{\Xs}$ with $T(\xs) = t$ and let $\mathtt{c}$ be a completion of $\es$ to $\xs$. We construct an $A$ with $\mathtt{c} \in \gamma(A)$ as follows. Since $\mathtt{c}$ is a completion we have that $\mathtt{c}(\es_i) = \xs_j$ implies $\es_i \sim \xs_j$ and, hence, we set $a_{i,j} = 0$ if $b_i \not\sim m_j$. For all other entries in $A$ we set $a_{i,j} = |\{ \es_j \mid \mathtt{c}(\es_j) = \xs_i\}|$. Since $\mathtt{c}$ is surjective, we have that $\sum_{j} a_{i,j} = c_i$ and since $\mathtt{c}$ is injective, we have that $\sum_{i} a_{i,j} = d_j$. Hence, $A \in \mathcal{A}_{t,\es}$.

To prove statement (3), let $A, A' \in \mathcal{A}_{t,\es}$ with $A \neq A'$. Since $A \neq A'$ we have that there exist $i,j$ such that, without loss of generality, $a_{i,j} < a'_{i,j}$. Hence, every $\mathtt{c} \in A$ maps fewer elements of type $m_j$ to elements of type $b_i$ than every $\mathtt{c}' \in A'$. Hence, $\mathtt{c} \neq \mathtt{c'}$ for every $\mathtt{c} \in A$ and every $\mathtt{c'} \in A'$.

To prove statement (4), let  $A \in \mathcal{A}_{t,\es}$. Every $\mathtt{c} \in A$ maps $a_{i,j}$ elements in $\{\es_1,...,\es_k\}$ of type $m_j$ to $a_{i,j}$ elements in $\{\xs_1,...,\xs_k\}$ of type $b_i$. Hence, the size of the set of completions represented by $A$ is, for each $1 \leq j \leq 3^w$, the number of different ways to place $a_{i,j}$ balls of color $i$, $1\leq i \leq 2^w$,  into $d_j$ urns. Hence, 
\begin{align}
  |\gamma(A)| = \prod_{j=1}^{3^w} \prod_{i=1}^{2^w} \binom{d_j-\sum_{q=1}^{i-1} a_{i,q}}{a_{i,j}}. \notag 
\end{align}

From the statements (1)-(4) we can conclude that   
\begin{align} 
|S_{t,\es}| = \sum_{A \in \mathcal{A}_{t,\es}} |\gamma(A)|. \notag 
\end{align}

This allows us to prove (b) and (c). We can construct $\mathcal{A}_{t,\es}$ in time polynomial in $n$ as follows. There are $2^w 3^w$ entries in a $2^w \times 3^w$ matrix and each entry has at most $k$ different values. Hence, we can enumerate all $k^{6^w} \leq n^{6^w}$ possible matrices $A \in \mathbb{M}(2^w,3^w)$.
We simply select those $A$ for which the conditions in the definition of $\mathcal{A}_{t,\es}$ hold. For one $A \in \mathcal{A}_{t,\es}$ we can efficiently construct one $\mathtt{c}$ and the $\xs$ that it completes $\es$ to. This proves (b). Finally, we compute $|S_{t,\es}|$. This proves~(c).
\end{proof}

\section{Proof of Theorem~\ref{theorem:monadic-decomposition}}

\begin{proof}
Let $P_1,...,P_N$ be the $N$ unary predicates of a given MLN and let $\dom = \{1,...,k\}$ be the domain. After grounding, there are $k$ ground atoms per predicate. We write $P_j(i)$ to denote the ground atom that resulted from instantiating predicate $P_j$ with domain element $i$.  Let $\mathcal{X} = \{\Xs_1,...,\Xs_k\}$ be a decomposition of the set of  ground atoms with $\Xs_i = \{P_1(i),P_2(i),...,P_N(i)\}$ for every $1 \leq i \leq k$. A renaming automorphism~\citep{hai2012automorphism,niepert2012lmcmc} is a permutation of the ground atoms that results from a permutation of the  domain elements. The joint distribution over all ground atoms remains invariant under these permutations. Consider the permutation of ground atoms that results from swapping two domain elements $i \leftrightarrow i'$. This permutation acting on the set of ground atoms permutes the components $\Xs_i$ and $\Xs_{i'}$ and leaves all other components invariant. Since this is possible for each pair $i,i' \in \{1,...,k\}$ it follows that the decomposition $\mathcal{X}$ is exchangeable.
\end{proof}

\section{Proof of Theorem~\ref{theorem:twologvar}}

\begin{proof}
Let $P_1,...,P_M$ be the $M$ unary predicates and let $Q_1,...,Q_N$ be the $N$ binary predicates of a given MLN and let $\dom = \{1,...,k\}$ be the domain. After grounding, there are $k$ ground atoms per unary and $k^2$ ground atoms per binary predicate. We write $P_{\ell}(i)$ to denote the ground atom that resulted from instantiating unary predicate $P_{\ell}$ with domain element $i$ and $Q_{\ell}(i,j)$ to denote the ground atom that resulted from instantiating binary predicate $Q_{\ell}$ with domain elements $i$ and $j$. 

Let $\Xs$ be the set of all ground atoms, let $\Ys = \{P_{\ell}(i) \mid 1 \leq i \leq k, 1\leq \ell \leq M\} \cup \{ Q_{\ell}(i,i) \mid 1 \leq i \leq k, 1 \leq \ell \leq N\}$ and let $\Zs = \Xs - \Ys$.
Moreover, let $\mathcal{Y} = \{\Ys_1,...,\Ys_k\}$ with $\Ys_i = \{P_{\ell}(i)  \mid 1 \leq \ell \leq M\} \cup \{ Q_{\ell}(i,i) \mid 1 \leq \ell \leq N\}$.
We can make the same arguments as in the proof of Theorem~\ref{theorem:monadic-decomposition} to show that $\mathcal{Y}$ is exchangeable. 

Now, we prove that the variables $\Zs$ are independently decomposable given~$\Ys$.
Let $\Zs_{i,j} = \{Q_1(i,j),...,Q_N(i,j),Q_1(j,i),...,Q_N(j,i)\}$ for all $1 \leq i < j \leq k$. Now, let $f$ be any ground formula and let $G$ be the set of ground atoms occurring in both $f$ and $\Zs$. Then, either $G \subseteq \Zs_{i,j}$ or $G \cap \Zs_{i,j} = \emptyset$ since every formula in the MLN has at most two variables. Hence, $\{\Zs_{i,j} \mid 1 \leq i < j \leq k\}$ is a decomposition of $\Zs$ with $\binom{k}{2}$ components, width $2N$, and $\Pr(\Zs,\ys)$ factorizes as
$$\Pr(\Zs,\ys) = \prod_{\substack{i,j\\i < j}} Q_{i,j}(\Zs_{i,j},\ys).$$
By the properties of MLNs, we have that the $Q_{i,j}(\Zs_{i,j},\ys)$ are computable in time exponential in the width of the decomposition but polynomial in $k$.
\end{proof}

\section{Proof of Theorem~\ref{theorem:twologvar-lifted}}

\begin{proof}
  Suppose that the query $\es$ contains a bounded number of binary atoms whose arguments are constants from the set~$K$.
  Consider the set of variables $\Qs$ consisting of all unary atoms whose argument comes from $K$, and all binary atoms whose arguments both come from $K$. 
  The unary atoms in $\Qs$ are no long marginally exchangeable, because their arguments can appear asymmetrically in~$\es$. 
  We can now answer the query by simply enumerating all states $\qs$ of $\Qs$ and performing inference in each $\Pr(\Ys,\Zs,\qs)$ separately, were all variables $\Ys$ have again become marginally exchangeable, and all variables $\Zs$ have become independently decomposable given $\Ys$. 
  The construction of $\Ys$ and $\Zs$ is similar to the proof of Theorem~\ref{theorem:twologvar}, except that some additional binary atoms are now in $\Ys$ instead of $\Zs$. These atoms have one argument in $K$, and one not in $K$, and are treated as unary.
  When we bound the number of binary atoms in the query, the size of $\Qs$ will not be a function of the domain size, and enumerating over all states $\qs$ is domain-lifted.
\end{proof}

\bibliographystyle{aaai}
\bibliography{references} 

\begin{thebibliography}{}

\bibitem[\protect\citeauthoryear{Boutilier \bgroup et al\mbox.\egroup
  }{1996}]{boutilier1996context}
Boutilier, C.; Friedman, N.; Goldszmidt, M.; and Koller, D.
\newblock 1996.
\newblock Context-specific independence in {Bayesian} networks.
\newblock In {\em Proceedings of the 12th Conference on Uncertainty in
  artificial intelligence (UAI)},  115--123.

\bibitem[\protect\citeauthoryear{Bui, Huynh, and de Salvo~Braz}{2012}]{BuiHB12}
Bui, H.; Huynh, T.; and de~Salvo~Braz, R.
\newblock 2012.
\newblock Exact lifted inference with distinct soft evidence on every object.
\newblock In {\em Proceedings of the 26th Conference on Artificial Intelligence
  (AAAI)}.

\bibitem[\protect\citeauthoryear{Bui, Huynh, and
  Riedel}{2012}]{hai2012automorphism}
Bui, H.; Huynh, T.; and Riedel, S.
\newblock 2012.
\newblock Automorphism groups of graphical models and lifted variational
  inference.

\bibitem[\protect\citeauthoryear{Darwiche}{2009}]{darwiche2009modeling}
Darwiche, A.
\newblock 2009.
\newblock {\em Modeling and reasoning with Bayesian networks}.
\newblock Cambridge University Press.

\bibitem[\protect\citeauthoryear{De~Raedt \bgroup et al\mbox.\egroup
  }{2008}]{DeRaedt2008-PILP}
De~Raedt, L.; Frasconi, P.; Kersting, K.; and Muggleton, S., eds.
\newblock 2008.
\newblock {\em Probabilistic inductive logic programming: theory and
  applications}.
\newblock Berlin, Heidelberg: Springer-Verlag.

\bibitem[\protect\citeauthoryear{{de Salvo Braz}, Amir, and
  Roth}{2005}]{DeSalvoBraz2005}
{de Salvo Braz}, R.; Amir, E.; and Roth, D.
\newblock 2005.
\newblock {Lifted first-order probabilistic inference}.
\newblock In {\em Proceedings of the International Joint Conference on
  Artificial Intelligence (IJCAI)},  1319--1325.

\bibitem[\protect\citeauthoryear{Diaconis and Freedman}{1980a}]{Diaconis:1980}
Diaconis, P., and Freedman, D.
\newblock 1980a.
\newblock De {F}inetti's generalizations of exchangeability.
\newblock In {\em Studies in Inductive Logic and Probability}, volume~II.

\bibitem[\protect\citeauthoryear{Diaconis and Freedman}{1980b}]{DnF:1980}
Diaconis, P., and Freedman, D.
\newblock 1980b.
\newblock Finite exchangeable sequences.
\newblock {\em The Annals of Probability} 8(4):745--764.

\bibitem[\protect\citeauthoryear{Getoor and Taskar}{2007}]{Getoor:2007}
Getoor, L., and Taskar, B.
\newblock 2007.
\newblock {\em Introduction to Statistical Relational Learning}.
\newblock The MIT Press.

\bibitem[\protect\citeauthoryear{Gogate and
  Domingos}{2011}]{gogate2012probabilistic}
Gogate, V., and Domingos, P.
\newblock 2011.
\newblock Probabilistic theorem proving.
\newblock In {\em Proceedings of the 27th Conference on Uncertainty in
  Artificial Intelligence (UAI)},  256--265.

\bibitem[\protect\citeauthoryear{Gupta, Diwan, and Sarawagi}{2007}]{Gupta:2007}
Gupta, R.; Diwan, A.~A.; and Sarawagi, S.
\newblock 2007.
\newblock Efficient inference with cardinality-based clique potentials.
\newblock In {\em Proceedings of the 24th International Conference on Machine
  Learning (ICML)},  329--336.

\bibitem[\protect\citeauthoryear{Jaeger and {Van den
  Broeck}}{2012}]{JaegerStarAI12}
Jaeger, M., and {Van den Broeck}, G.
\newblock 2012.
\newblock Liftability of probabilistic inference: {U}pper and lower bounds.
\newblock In {\em Proceedings of the 2nd International Workshop on Statistical
  Relational AI}.

\bibitem[\protect\citeauthoryear{Jha \bgroup et al\mbox.\egroup
  }{2010}]{Jha2010}
Jha, A.; Gogate, V.; Meliou, A.; and Suciu, D.
\newblock 2010.
\newblock Lifted inference seen from the other side: The tractable features.
\newblock In {\em Proceedings of the 24th Conference on Neural Information
  Processing Systems (NIPS)}.

\bibitem[\protect\citeauthoryear{Kersting}{2012}]{Kersting:2012}
Kersting, K.
\newblock 2012.
\newblock Lifted probabilistic inference.
\newblock In {\em Proceedings of European Conference on Artificial Intelligence
  (ECAI)}.

\bibitem[\protect\citeauthoryear{Koller and Friedman}{2009}]{Koller:2009}
Koller, D., and Friedman, N.
\newblock 2009.
\newblock {\em Probabilistic Graphical Models}.
\newblock The MIT Press.

\bibitem[\protect\citeauthoryear{Lauritzen \bgroup et al\mbox.\egroup
  }{1984}]{lauritzen:1984}
Lauritzen, S.~L.; Barndorff-Nielsen, O.~E.; Dawid, A.~P.; Diaconis, P.; and
  Johansen, S.
\newblock 1984.
\newblock Extreme point models in statistics.
\newblock {\em Scandinavian Journal of Statistics} 11(2).

\bibitem[\protect\citeauthoryear{Lauritzen}{1988}]{Lauritzen:1988}
Lauritzen, S.~L.
\newblock 1988.
\newblock {\em Extremal families and systems of sufficient statistics}.
\newblock Lecture notes in statistics. Springer-Verlag.

\bibitem[\protect\citeauthoryear{Lauritzen}{1996}]{Lauritzen:1996}
Lauritzen, S.~L.
\newblock 1996.
\newblock {\em Graphical Models}.
\newblock Oxford University Press.

\bibitem[\protect\citeauthoryear{Milch \bgroup et al\mbox.\egroup
  }{2008}]{milch2008lifted}
Milch, B.; Zettlemoyer, L.; Kersting, K.; Haimes, M.; and Kaelbling, L.
\newblock 2008.
\newblock Lifted probabilistic inference with counting formulas.
\newblock {\em Proceedings of the 23rd AAAI Conference on Artificial
  Intelligence}  1062--1068.

\bibitem[\protect\citeauthoryear{Mladenov and
  Kersting}{2013}]{mladenov2013lifted}
Mladenov, M., and Kersting, K.
\newblock 2013.
\newblock Lifted inference via k-locality.
\newblock In {\em Proceedings of the 3rd International Workshop on Statistical
  Relational AI}.

\bibitem[\protect\citeauthoryear{Niepert and Domingos}{2014}]{niepert:2014b}
Niepert, M., and Domingos, P.
\newblock 2014.
\newblock Exchangeable variable models.
\newblock In {\em Proceedings of the International Conference on Machine
  Learning (ICML)}.

\bibitem[\protect\citeauthoryear{Niepert}{2012a}]{niepert2012lmcmc}
Niepert, M.
\newblock 2012a.
\newblock Lifted probabilistic inference: An {MCMC} perspective.
\newblock In {\em Proceedings of StaRAI}.

\bibitem[\protect\citeauthoryear{Niepert}{2012b}]{niepertorbits}
Niepert, M.
\newblock 2012b.
\newblock {Markov} chains on orbits of permutation groups.
\newblock In {\em Proceedings of the 28th Conference on Uncertainty in
  Artificial Intelligence (UAI)}.

\bibitem[\protect\citeauthoryear{Niepert}{2013}]{niepert2013symmetry}
Niepert, M.
\newblock 2013.
\newblock Symmetry-aware marginal density estimation.
\newblock In {\em Proceedings of the 27th Conference on Artificial Intelligence
  (AAAI)}.

\bibitem[\protect\citeauthoryear{Noessner, Niepert, and
  Stuckenschmidt}{2013}]{noessner:2013}
Noessner, J.; Niepert, M.; and Stuckenschmidt, H.
\newblock 2013.
\newblock {RockIt: Exploiting Parallelism and Symmetry for MAP Inference in
  Statistical Relational Models}.
\newblock In {\em Proceedings of the 27th Conference on Artificial Intelligence
  (AAAI)}.

\bibitem[\protect\citeauthoryear{Pearl}{1988}]{Pearl:1988}
Pearl, J.
\newblock 1988.
\newblock {\em Probabilistic reasoning in intelligent systems: networks of
  plausible inference}.
\newblock Morgan Kaufmann.

\bibitem[\protect\citeauthoryear{Poole}{2003}]{Poole2003}
Poole, D.
\newblock 2003.
\newblock First-order probabilistic inference.
\newblock In {\em Proceedings of the International Joint Conference on
  Artificial Intelligence (IJCAI)},  985--991.

\bibitem[\protect\citeauthoryear{Richardson and
  Domingos}{2006}]{richardson2006markov}
Richardson, M., and Domingos, P.
\newblock 2006.
\newblock Markov logic networks.
\newblock {\em Machine learning} 62(1-2):107--136.

\bibitem[\protect\citeauthoryear{Robertson and
  Seymour}{1986}]{robertson1986graph}
Robertson, N., and Seymour, P.
\newblock 1986.
\newblock Graph minors. {II.} {A}lgorithmic aspects of tree-width.
\newblock {\em Journal of algorithms} 7(3):309--322.

\bibitem[\protect\citeauthoryear{Sheldon and Dietterich}{2011}]{Sheldon:2011}
Sheldon, D., and Dietterich, T.
\newblock 2011.
\newblock Collective graphical models.
\newblock In {\em Advances in Neural Information Processing Systems (NIPS)}.
\newblock  1161--1169.

\bibitem[\protect\citeauthoryear{Taghipour \bgroup et al\mbox.\egroup
  }{2012}]{taghipour2012lifted}
Taghipour, N.; Fierens, D.; Davis, J.; and Blockeel, H.
\newblock 2012.
\newblock Lifted variable elimination with arbitrary constraints.
\newblock In {\em Proceedings of the fifteenth international conference on
  Artificial Intelligence and Statistics}, volume~22,  1194--1202.

\bibitem[\protect\citeauthoryear{Taghipour \bgroup et al\mbox.\egroup
  }{2013}]{taghipour2013completeness}
Taghipour, N.; Fierens, D.; Van~den Broeck, G.; Davis, J.; and Blockeel, H.
\newblock 2013.
\newblock Completeness results for lifted variable elimination.
\newblock In {\em Proceedings of the 16th Conference on Artificial Intelligence
  and Statistics},  572--580.

\bibitem[\protect\citeauthoryear{{Van den Broeck} and
  Darwiche}{2013}]{VdBDarwiche13}
{Van den Broeck}, G., and Darwiche, A.
\newblock 2013.
\newblock On the complexity and approximation of binary evidence in lifted
  inference.
\newblock In {\em Advances in Neural Information Processing Systems 26 (NIPS)}.

\bibitem[\protect\citeauthoryear{{Van den Broeck} and Davis}{2012}]{VdBAAAI12}
{Van den Broeck}, G., and Davis, J.
\newblock 2012.
\newblock Conditioning in first-order knowledge compilation and lifted
  probabilistic inference.
\newblock In {\em Proceedings of the 26th Conference on Artificial Intelligence
  (AAAI)}.

\bibitem[\protect\citeauthoryear{{Van den Broeck} \bgroup et al\mbox.\egroup
  }{2011}]{VdBIJCAI11}
{Van den Broeck}, G.; Taghipour, N.; Meert, W.; Davis, J.; and De~Raedt, L.
\newblock 2011.
\newblock Lifted probabilistic inference by first-order knowledge compilation.
\newblock In {\em Proceedings of the 22nd International Joint Conference on
  Artificial Intelligence (IJCAI)},  2178--2185.

\bibitem[\protect\citeauthoryear{{Van den Broeck}}{2011}]{VdBNIPS11}
{Van den Broeck}, G.
\newblock 2011.
\newblock On the completeness of first-order knowledge compilation for lifted
  probabilistic inference.
\newblock In {\em Advances in Neural Information Processing Systems 24
  (NIPS),},  1386--1394.

\end{thebibliography}

\end{document}